\documentclass[twoside]{article}
\usepackage{aistats2016}

\usepackage{hyperref}
\usepackage{url}
\usepackage{graphicx} 
\usepackage{subfigure} 
\usepackage[numbers]{natbib}
\usepackage{hyperref}
\usepackage{enumitem}
\usepackage{amsmath,amssymb,amsfonts,amsthm,xfrac}
\usepackage{capt-of}
\usepackage[ruled,vlined,linesnumbered]{algorithm2e}
\newtheorem{theorem}{Theorem}[section]

\theoremstyle{definition}

\theoremstyle{remark}

\usepackage{xcolor}

\DeclareMathOperator*{\argmin}{arg\,min}

\begin{document}

%

%

\twocolumn[

\aistatstitle{Embarrassingly Parallel Variational Inference \\in Nonconjugate Models}

\aistatsauthor{ Willie Neiswanger \And Chong Wang \And Eric Xing }

\aistatsaddress{
Machine Learning Department\\
Carnegie Mellon University\\
\texttt{willie@cs.cmu.edu}
\And 
AI Lab\\
Baidu Research\\
\texttt{chongw@cs.princeton.edu}
\And 
School of Computer Science\\
Carnegie Mellon University \\
\texttt{epxing@cs.cmu.edu}
} ]

\begin{abstract}
We develop a parallel variational inference (VI) procedure
for use in data-distributed settings, where each machine
only has access to a subset of data and runs VI
independently, without communicating with other machines.
This type of ``embarrassingly parallel'' procedure has
recently been developed for MCMC inference algorithms;
however, in many cases it is not possible to directly extend
this procedure to VI methods without requiring 
certain restrictive exponential family 
conditions on the form of the model. Furthermore, most
existing (nonparallel) VI methods are restricted to use on
conditionally conjugate models, which limits their
applicability.  To combat these issues, we make use of the
recently proposed nonparametric VI to facilitate an
embarrassingly parallel VI procedure that can be 
applied to a wider scope of models, including to 
nonconjugate models. We derive our embarrassingly parallel
VI algorithm, analyze our method theoretically, and
demonstrate our method empirically on a few nonconjugate
models.
\end{abstract}

\vspace{-2mm}
\section{Introduction}
\label{introduction}
Many large, modern datasets are collected and stored in a
distributed fashion by multiple sensors or data-collecting
agents. Examples of this include medical data recorded in
hospitals throughout a country, weather data gathered by a
collection of sensors, web data scraped by a network of
machines, and cell phone data collected on users' phones.
Inference algorithms that can operate in these distributed
settings---by processing subsets of data separately and in
parallel---are particularly advantageous. This is because
they mitigate the need for transfering data to a central
location for analysis, reduce both the memory usage and
computation time of inference
\cite{low2014graphlab,neiswanger2013asymptotically}, allow
for continuous data collection from independently operating
agents \cite{campbell2014approximate}, and allow for
sensitive data to be processed independently in secure
locations (which can yield privacy guarantees
\cite{nissim2007smooth}).

Variational inference (VI) methods are general procedures
for approximate inference in Bayesian models, and they have
been applied successfully in a wide variety of domains
\cite{jordan1999introduction,beal2003variational}. This
paper is concerned with developing better VI methods for use
in distributed settings. One major issue with most existing
parallel methods is that they often require synchronization
between machines at regular intervals
\cite{zhai2012mr,nallapati2007parallelized,yan2009parallel}.
Communication between machines due to this synchronization
can greatly reduce the efficiency of these procedures, as
each machine must wait for information from other machines
before proceeding with computation. Furthermore,
communication requirements may increase the difficulty of
system implementation and maintenance
\cite{low2014graphlab}, and necessitate the transfer of
(potentially sensitive) data between machines
\cite{williamson2013parallel}.

We aim to develop a new ``embarrassingly parallel'' algorithm
for VI in data-distributed settings, which is a type of
parallel algorithm where there is no regular communication
between machines. Given a dataset partitioned over a
collection of machines, embarrassingly parallel VI methods 
carry out the following two steps:
\vspace{-2mm}
\begin{enumerate}[itemsep=1pt]
\item Perform variational inference on the subset of data on
each machine in parallel (independently, without
communication between machines).
\item Combine the results from all machines to yield a
variational inference result for the full-data posterior
distribution.
\end{enumerate}
\vspace{-2mm}
These two steps are only performed once, and there is only
communication between machines at one point in the second
step, when collecting results from each of the local
instances of VI.

Recently, progress has been made toward this goal for mean
field variational inference methods limited to models
with certain exponential family restrictions on the
likelihood and prior distribution 
\cite{broderick2013streaming,campbell2014approximate}.
These methods use a decomposition of the posterior that
takes advantage of closedness properties of exponential
family densities under products and quotients. However,
these modeling assumptions are fairly restrictive, and this
decomposition cannot be applied to many popular models
(including logistic regression, correlated topic models, and
nonlinear matrix factorization models). Additionally, this
approximation family is typically inadequate to capture
multimodal densities \cite{gershman2012nonparametric}.

A separate line of work has aimed to develop
``nonconjugate'' variational inference methods for models
without tractable exponential family conditional
distributions
\cite{gershman2012nonparametric,wang2013variational}.
Similar to these methods, we would like a general inference
algorithm that can be applied to a wide class of Bayesian
models, yet operates in this embarrassingly parallel
setting. However, the variational families employed by these
nonconjugate methods are not in a form that allows us to
apply the above-mentioned decomposition strategy for 
parallelization.

Recent papers in the Markov chain Monte Carlo (MCMC)
literature have introduced an alternative decomposition of
the posterior for parallel inference
\cite{Scott:2013,neiswanger2013asymptotically,wang2013parallel},
which involves the product of so called \emph{subposterior
densities} (i.e. posterior densities given a subset of data
with an underweighted prior). We apply this new
decomposition to a nonconjugate variational inference method
called nonparametric variational inference (NVI)
\cite{gershman2012nonparametric} to perform
low-communication, parallel inference in a general class of
models. In particular, we only require weak
differentiability conditions on the joint log probability.

The main contribution of our method is that it provides a
way to perform embarrassingly parallel inference in data
distributed settings for a more-general class of 
Bayesian models without requiring conditional conjugacy 
or exponential family assumptions on the 
model likelihood or prior. In the
following sections, we derive the posterior decomposition
used by our method, show how we can combine local
nonparametric variational approximations to form a
variational approximation to the full-data posterior
density, and analyze the computational complexity of our
algorithms.  Finally, we demonstrate our method empirically
on a few nonconjugate Bayesian models.

\section{Preliminaries}
We describe existing work on embarrassingly parallel
VI with exponential family restrictions, existing work on 
embarrassingly parallel MCMC, and the difficulties with
extending these methods to variational inference in more 
general, nonconjugate models.

Suppose we have a large set of $N$ i.i.d. data points, $x^N
= \{x_1,\ldots,x_N\}$, a likelihood for these data
parameterized by $\theta \in \mathbb{R}^d$, written $p(x^N |
\theta)$, and a prior density for $\theta$, written $p(\theta)$. We
can write the posterior density given all $N$ data points
(which we will also refer to as the ``full-data'' posterior)
as 
\vspace{-2mm}
\begin{align}
    p(\theta|x^N) \propto p(\theta)p(x^N|\theta) =
    p(\theta)\prod_{i=1}^N p(x_i|\theta).
\end{align}
Now suppose the data $x^N$ is partitioned into $M$ subsets 
$\{ x^{n_1},\ldots,x^{n_M} \}$ of sizes $n_1,\ldots,n_M$,
and distributed over $M$ machines. Recent works in
embarrassingly parallel VI \cite{broderick2013streaming,
campbell2014approximate} have proposed the following
solution for inference in this setting.  First, in an
embarrassingly parallel fashion, compute
\vspace{-2mm}
\begin{align}
    q_1^*,\ldots,q_M^* = \argmin_{q_1,\ldots,q_M}
    \sum_{m=1}^M \text{KL}\left[q_m(\theta) ||
    p(\theta|x^{n_m})\right]
\end{align}
where $p(\theta|x^{n_m})$ is the posterior given a subset of
data $x^{n_m}$.  Second, form the full-data posterior
variational approximation with
\vspace{-2mm}
\begin{align}
    \label{eq:classicEPVI}
    q^*(\theta) \propto \left( \prod_{m=1}^M q_m^*(\theta) \right) /
    p(\theta)^{M-1}.
\end{align}
The justification for this solution is that the full-data
posterior can be decomposed as $p(\theta|x^N)$
$\propto$ $\left( \prod_{m=1}^M p(\theta|x^{n_m}) \right) /
p(\theta)^{M-1}$, and further, it can be shown that
the above objective retains an important property of the
classic (nonparallel) KL objective: if the objective is zero
then the full-data approximation $q^*(\theta)$ is equal to
the full-data posterior $p(\theta|x^N)$. I.e., if 
$\sum_{m=1}^M \text{KL}[q_m(\theta) || p(\theta|x^{n_m})]$
$=$ $0$ $\implies$ $\text{KL}[q^*(\theta)||p(\theta|x^N)]$
$=$ $0$ $\implies$ $q^*(\theta)$ $=$ $p(\theta|x^N)$.

However, this solution has a few major restrictions on the
form of the variational approximation and model.  Namely, to
form $q^*$, these methods must tractably compute the product
of the $M$ variational approximations divided by the prior
density (equation~(\ref{eq:classicEPVI})). These methods do this by
limiting their scope to conditionally conjugate exponential
family models, and then using mean field variational methods
that restrict the variational approximation to the same
exponential family as the prior.

To attempt to extend the scope of models to which
embarrassingly parallel VI methods can be applied, we turn
to a separate line of work on embarrassingly parallel MCMC
methods \cite{Scott:2013,neiswanger2013asymptotically,
wang2013parallel}), which use an alternative decomposition
of the posterior distribution.  Let the $m^{\text{th}}$
\emph{subposterior} density, $p_m(\theta)$, be defined as
the posterior given the $m^\text{th}$ data subset with an
underweighted prior, written $p_m(\theta) =
p(\theta)^{\frac{1}{M}}p(x^{n_m}|\theta)$. This is defined
such that the product of the $M$ subposterior densities is
proportional to the full-data posterior, i.e.
\begin{align}
    p_1 \cdots p_M (\theta)  \propto p(\theta)\prod_{m=1}^M
    p(x^{n_m}|\theta) \propto p(\theta | x^N).
\end{align}
In these methods, a subposterior density estimate
$\widehat{p}_m(\theta)$ is learned on each machine 
(via sampling), and the product of these estimates 
$\prod_{m=1}^M \widehat{p_m}(\theta)$ yields an
approximation of the full-data posterior density.

However, we cannot directly apply this new decomposition to
typical mean field variational inference approximations (as
is done in embarrassingly parallel VI) for the following two
reasons: 
\vspace{-2mm}
\begin{enumerate}[itemsep=1pt]
\item The underweighted prior $p(\theta)^\frac{1}{M}$ in the
subposterior density may lose conjugacy necessary for the
requisite exponential-family-conditionals. Hence, it may not
be easy to directly apply these VI methods to approximate
the subposterior.
\item Even if we are able to learn a variational
approximation for each subposterior, the product of
subposterior variational approximations may not have a
tractable form that we can analytically compute.
\end{enumerate}
\vspace{-2mm}
Therefore, to use this alternative decomposition to apply VI
to a broader scope of models, we need a family of variational
approximations that can be run on general subposterior 
densities (including those of nonconjugate models)
while maintaining a tractable density product that can be
analytically computed.



\section{Embarrassingly Parallel Variational Inference 
in Nonconjugate Models}
Embarrassingly parallel variational inference (EPVI)
in nonconjugate models is a
parallel approximate Bayesian inference method for
continuous posterior distributions. It is generally
applicable, requiring only that the first two derivatives of
the log-joint probability density are computable. For a
dataset partitioned over $M$ machines, VI 
is run in parallel on each machine to approximate the $M$
subposterior densities; afterwards, the local subposterior
approximations are combined by computing their product,
which approximates the full-data posterior density. Each
machine performs variational inference without sharing
information, in an embarrassingly parallel manner. We
summarize this procedure in Algorithm~\ref{epvi_algorithm}.

\begin{algorithm}[!ht]
    \caption{Embarrassingly Parallel Variational Inference in Nonconjugate Models} 
    \label{epvi_algorithm}
    \KwIn{Partitioned dataset $\{x^{n_1},\ldots,x^{n_M}\}$.}
    \KwOut{Variational approximation $q^*(\theta)$ for the
    full-data posterior density $p(\theta | x^N)$.}
    \vspace{2pt}
    \ForPar{$m=1,\ldots,M$}{
        Learn a variational approximation $q_m^*(\theta)$
        for the $m^{th}$ subposterior $p_m(\theta)$, given
        data $x^{n_m}$.
    }
    Compute product $\prod_{m=1}^M q_m^*(\theta)$ of
    subposterior approximations to yield the full-data
    variational approximation $q^*(\theta)$.
\end{algorithm}

We illustrate our EPVI procedure for a Bayesian logistic
regression model on a toy dataset in
Figure~\ref{fig:epviIllustration}.

\begin{figure*}[!ht]
    \center{\includegraphics[width=1\textwidth]{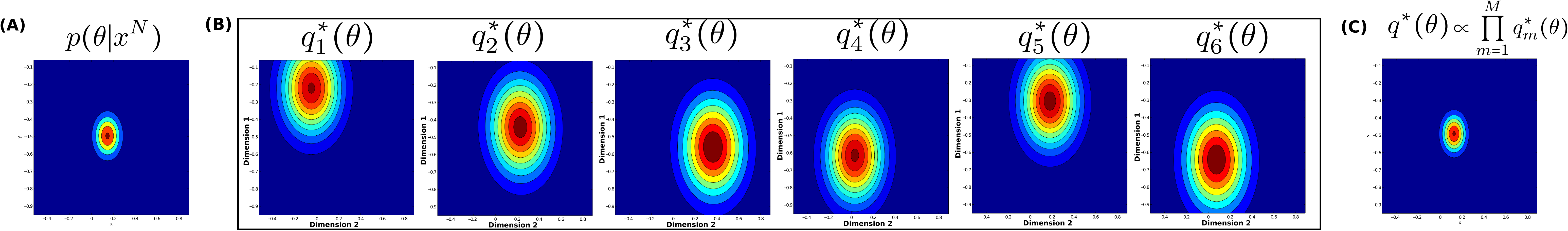}}
    \caption{\label{fig:epviIllustration} Illustration of our
    embarrassingly parallel VI method,
    shown for a Bayesian logistic regression model on a toy
    dataset. In (a) we show the first two dimensions of the
    full-data posterior density. In (b) we show the first
    two dimensions of each of the $M=6$ subposterior
    variational approximations after running VI on 
    each subset of data independently. In (c)
    we show the first two dimensions of the combined product
    density (formed using the six subposterior variational
    approximations), which recovers the posterior shown in
    (a).}
\end{figure*}

\subsection{EPVI with Nonparametric Variational Inference}
In a recently proposed method known as nonparametric
variational inference (NVI)
\cite{gershman2012nonparametric}, a posterior approximation
is selected from a variational family of densities of the
form $q(\theta) = \frac{1}{K} \sum_{k=1}^K
\mathcal{N}_d(\theta | \mu_k, \sigma_k^2 I_d)$.  Some
advantages of this method are that it can capture multimodal
posterior distributions, can be applied to many
nonconjugate models (in fact, the only requirement is that
the first two derivatives of the log joint probability are
computable), and has an efficient algorithm to optimize the
variational objective. In our case, NVI allows us to perform
variational inference on subposterior densities without
worrying if they retain the conjugacy necessary to easily
apply typical mean-field approximations, and also allows us
to develop a method to combine the subposterior variational
approximations (i.e.  allows us to derive an analytic form
for the product of these approximations) to produce a
full-data posterior variational approximation. After running
this procedure on a subset of data $x^{n_m}$ on a machine
$m$, we can write the inferred variational approximation for
the subposterior distribution as 
\vspace{-1mm}
\begin{align}
q_m^*(\theta) = \frac{1}{K} \sum_{k=1}^K
\mathcal{N}_d(\theta | \mu_k^{(m)}, \sigma_k^{2^{(m)}} I_d).
\end{align}
Due to this choice of $q_m^*$, we have an analytic form for
the product of these densities, $\prod_{m=1}^M
q^*_m(\theta)$, which gives us a variational approximation
for the subposterior density product (and hence for the
full-data posterior). In particular, the product of these
$M$ mixture-of-Gaussians variational densities gives a
(non-uniformly weighted) mixture-of-Gaussians density with
$K^M$ components. We can write this product mixture as
\vspace{-2mm}
\begin{align}
    q^*(\theta) &\propto \prod_{m=1}^M q^*_m(\theta) =
    \frac{1}{K^M} \prod_{m=1}^M  \sum_{k_m = 1}^K
    \mathcal{N}_d \left( \theta | \mu_{k_m}^{(m)},
    \sigma^{2^{(m)}}_{k_m} I_d \right) \nonumber \\ &=
    \sum_{k_1=1}^K \cdots \sum_{k_M=1}^K w_{k\cdot}
    \mathcal{N}_d (\theta | \mu_{k\cdot}, \sigma_{k\cdot}^2
    I_d)  \label{mixDenProd}
\end{align}
where we use $k\cdot = (k_1,\ldots,k_M)$ to denote the vector
of M subposterior-component-indices (one from each
subposterior variational approximation mixture) associated
with a given component in this product mixture, and where
\vspace{-2mm}
\begin{align}
    \sigma_{k\cdot}^2 &= \left( \sum_{m=1}^M
    \left(\sigma_{k_m}^{2^{(m)}} \right)^{-1} \right)^{-1}\\
    \mu_{k\cdot} &= \sigma_{k\cdot}^2 I_d \left(
    \sum_{m=1}^M 
    \left( \left( \sigma^{2^{(m)}}_{k_m} \right)^{-1} I_d
    \right) \mu_{k_m}^{(m)}  \right)\\
    w_{k\cdot} &= \frac{ \prod_{m=1}^M \mathcal{N}_d 
	\left( \mu_{k_m}^{(m)} | \mu_{k\cdot}, 
        \sigma^{2^{(m)}}_{k_m} I_d  \right) }
	{ \mathcal{N}_d\left( \mu_{k\cdot} | \mu_{k\cdot}, 
        \sigma^{2}_{k\cdot} \right) }
\end{align}

\subsection{Computing the Variational Density Product
Mixture}
After learning the optimal local parameters $\{
\mu_{k_m}^{(m)}, \sigma_{k_m}^{2^{(m)}} \}_{k_m=1}^K$ for each of 
the $m \in \{1,\ldots,M\}$ subposteriors, we wish to form a
variational approximation to the full-data posterior density
by taking the product of the $M$ mixtures. However,
computing the parameters and weights for all 
$K^M$ components in the product mixture becomes infeasible as 
$M$ grows.

We typically perform Bayesian inference in order to compute
expectations with respect to, and explore, the posterior
distribution. In practice, one common way to achieve this is
to sample from the posterior, and then compute a sample 
expectation; this is done in both MCMC methods and
in VI methods (in the latter case, to compute expectations 
with respect to a
variational approximation after VI has finished runnning
\cite{gershman2012nonparametric,blei2004variational,
ranganath2013black}). Hence, 
instead of computing the product mixture (and afterwards, 
sampling from it to compute expectations), 
our solution is to bypass this step and 
directly generate samples from the product mixture. 
We give a procedure that allows us to compute 
expectations with respect to the variational approximation 
in this common sampling manner without requiring us 
to actually compute the variational approximation.

We give our method for sampling from the full-data
variational approximation in
Algorithm~\ref{sampleProductComponents}, and then prove that
this yields correct samples. The intuitive idea behind our
algorithm is the following. To sample from a mixture, one can
first sample a component index (proportional to the
component weights) and then sample from the chosen
mixture component. We therefore need a way to sample product 
mixture components (proportional to their weights) without
first computing all of the $K^M$ component weights. Our solution
is to form a Markov chain over the product mixture component
indices, and prove that its stationary distribution is a
categorical distribution with probability mass values
proportional to the product mixture component
weights.  Hence, at each step in this Markov chain, we can produce 
a sample from the full variational approximation while only
needing to compute a single new product mixture component.

\vspace{-2mm}
\begin{algorithm}[!ht]
    \caption{Markov chain for sampling variational density product
    mixture components} 
    \label{sampleProductComponents}
    \KwIn{Number of samples $R$, number of
    burn-in steps $b$, learned subposterior variational
    approximations $\{q_1^*(\theta),\ldots,q_M^*(\theta)
    \}$.}
    \KwOut{Parameters $\{ \mu_r, \sigma_r^2\}_{r=1}^R$ for the
    $R$ sampled product mixture components.}
    \vspace{2pt}
    Draw $k\cdot$ $=$ $(k_1,\ldots,k_M)
        \stackrel{\text{iid}}{\sim}
        \text{Unif}(\{1,\ldots,K\})$\tcc*{Initialize Markov chain} 
    \For{$s = 1,\ldots,b+R$}{
        Draw $m \sim \text{Unif}(\{1,\ldots,M\})$\\
        Set $c\cdot = (c_1,\ldots,c_M) \leftarrow k\cdot$\\
        Draw $c_m \sim
            \text{Unif}(\{1,\ldots,K\})$\\
        Draw $u \sim \text{Unif}([0,1])$;\\
        \If{$u < w_{c\cdot}/w_{k\cdot}$}{
            Set $k\cdot \leftarrow c\cdot$\\
        }
        \If{$s > b$} {
        Set $\mu_{s-b} \leftarrow \mu_{k\cdot}$\tcc*{Compute
            mean of sampled mixture component}
        Set $\sigma_{s-b}^2 \leftarrow
            \sigma_{k\cdot}^2$\tcc*{Compute var of sampled
            mixture component}
        }
    }
\end{algorithm}
Note that in Algorithm~\ref{sampleProductComponents}, at each step 
in the Markov chain, we perform two simple steps to sample the next 
product mixture component: we select a subposterior uniformly at random  
(line 3), and then re-draw one of its $K$ components uniformly at 
random (line 5); this specifies a new product mixture component. 
We then compute the ratio of the weight of this new 
product mixture component with the previous component's weight 
(line 7) and accept or reject this proposal (line 8). We then compute the 
parameters of the sampled component (lines 10-11).

\textbf{Correctness of Algorithm~\ref{sampleProductComponents}.} 
We prove that Algorithm~\ref{sampleProductComponents} defines a
Markov chain whose stationary distribution is the
distribution over the $K^M$ components in the product
mixture density.

\begin{theorem}
The procedure given by Algorithm~\ref{sampleProductComponents} defines
a Markov chain whose stationary distribution is the categorical
distribution (over $K^M$ categories) with category-probability 
parameter equal to the vector of product mixture component weights.
\end{theorem}
\begin{proof}
\vspace{-2mm}
Note that each of the $K^M$ product mixture components 
is associated with an M-dimensional vector
$k\cdot = (k_1,\ldots,k_M) \in \{1,\ldots,K\}^M$ 
(where $k_m$ denotes the index of the $m^\text{th}$ subposterior's
component that contributed a factor to this product mixture component). 
Hence, instead of sampling an index from a categorical 
distribution, we can equivalently view our
task as sampling an M-dimensional vector from a joint 
distribution over the space $ \{1,\ldots,K\}^M$, where each element
in this space has a probability mass proportional to its associated
product mixture component weight. We can therefore perform
Gibbs sampling over this space, where we sample from the 
conditional distribution over a subposterior component 
index $k_m$ given all other component indices.
To compute and then sample from this conditional distribution,
we could iterate over the $K$ possible values of $k_m$ 
(and compute the component weight of each); however, this could 
potentially be expensive for large $K$. Instead, we sample
one of the $K$ values for $k_m$ uniformly at random (line~5),
and our algorithm becomes a Metropolis-within-Gibbs algorithm 
\cite{gilks1995adaptive} where we've used an independent Metropolis 
proposal \cite{giordani2010adaptive,atchade2005improving} 
(which we achieve by accepting or 
rejecting, in lines~6-8, the independent Metropolis proposal 
made in line~5). Note that the dimension $m$ along which we take 
a Gibbs sampling step is chosen in line 3.
Since this Metropolis-within-Gibbs algorithm 
has been shown to have the correct stationary 
distribution \cite{gilks1995adaptive}, our proof is complete.
\end{proof}

\vspace{-2mm}
We describe the complexity of
Algorithm~\ref{sampleProductComponents} in
Section~\ref{sec:complexity}. In
Section~\ref{sec:experiments}, we verify that this
algorithm achieves the same results as taking expectations
after computing the mixture product exactly, while
drastically speeding-up performance.

\vspace{-1mm}
\paragraph{Sequential subposterior subset products.}
In some cases, it may be simpler to sample from the product
mixture in a sequential fashion by sampling from the product
of only a few subposteriors multiple times: we
first sample $R$ components from the product of groups of
$\tilde{M} < M$ approximations, and then repeat this process 
on the resulting (uniform) mixtures formed by the sampled 
components. This continues until samples from only one mixture
remain. For example, one could begin by sampling components from
the product of all $\frac{M}{2}$ pairs (leaving one subposterior
approximation alone if $M$ is odd), thereby forming 
$\frac{M}{2}$ uniformly weighted mixtures comprised of the
sampled components. This process is then repeated---forming 
pairs and sampling from the pair product 
mixture---until there are only samples from one product mixture 
remaining (which are approximate samples from the
full-data posterior). This method is potentially advantageous 
because each intermediate round of product mixture sampling could be done 
in parallel. However, more samples are potentially required from each
intermediate round to generate valid samples from the 
full variational approximation at the final product. We compare the 
effectiveness of this method in Section~\ref{sec:experiments}.

\vspace{-2mm}
\subsection{Method Complexity}
\label{sec:complexity}
\vspace{-2mm}
Consider a dataset with $N$ observations, partitioned over
$M$ machines.
Assume we have approximated each subposterior using NVI 
with $K$ components, where each component is defined 
by a $d$-dimensional parameter.
Computing all components of the product mixture exactly 
requires $O(dMK^M)$ operations. 
Computing $R$ samples from the product mixture approximation via
Algorithm~\ref{sampleProductComponents} requires $O(dRM)$ 
operations (assuming a constant number $b$ of burn-in steps). 
Computing sequential subposterior subset product samples with $R$ 
samples at each intermediate product requires $O(d R M^2)$
operations overall, but this could be reduced to $O(d R \log M)$ 
operations on a single machine if each of the $O(\log M)$ rounds of sampling 
are done in parallel.

Each machine learns and then communicates the optimal
variational parameters, which consist of $K$ mean parameter
vectors (each in $d$ dimensions), $K$ variance parameter
scalars, and $K$ weight parameter scalars. In total,
$MK(d+2)$ scalars are communicated throughout the entire
procedure.





\begin{figure*}[!ht]
    \center{\includegraphics[width=0.9\textwidth]{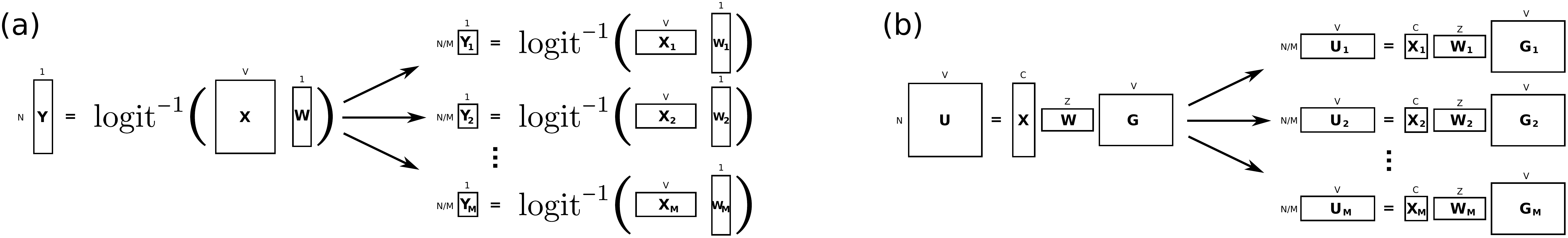}}
    \caption{\label{fig:modelDiagram} Diagrams of the data
    partitioning schemes for the (a) hierarchical logistic
    regression model and (b) topographic latent source
    analysis model.}
\end{figure*}

\vspace{-2mm}
\subsection{Method Scope}
\label{sec:scope}
\vspace{-2mm}
The algorithms described in this paper hold for posteriors distributions
with twice-differentiable densities in finite-dimensional 
real spaces. This method may be applied to nonconjugate 
Bayesian models, and models with multimodal posteriors, 
with little further restriction on the form of the 
model and prior distribution.

However, there are certain model types for which our 
method is not appropriate. These include discrete models, 
continuous posterior distributions over the simplex, and 
infinite dimensional models. Furthermore, our method 
may not be well suited to posteriors with high 
correlations or different scales between dimensions, 
and multimodal models suffering from label switching.

\vspace{-3mm}
\section{Empirical Study}
\label{sec:experiments}
\vspace{-2mm}
We demonstrate empirically the ability of our method to
significantly speed up VI on nonconjugate models in a 
distributed setting, while maintaining an accurate posterior
approximation. In particular, our experiments aim to show that:
\vspace{-3mm}
\begin{enumerate}
\item We can perform inference in a fraction of the time of
nonparallel VI methods as we increase M.
\item We can achieve similar performance as nonparallel VI
methods as we increase M.
\item Expectations computed via our product mixture sampling
method (Algorithm~\ref{sampleProductComponents}) achieve
similar performance as those computed via exact computation
of the product mixture.
\end{enumerate}
\vspace{-2mm}
To demonstrate these, we conduct experimental comparisons 
with the following strategies:
\begin{itemize}[nosep]
\item \textbf{Full-data nonparametric variational inference}
    (\texttt{NVI})---A (nonparallel) variational inference
    method designed for use in nonconjugate models, which we
    run on the full dataset. This method takes as a
    parameter the number of mixture components $K$.
\item \textbf{Subposterior inference on data subsets}
    (\texttt{Subposteriors})---The subposterior variational
    approximations, run on subsets of data. This method
    takes as a parameter the number of mixture components
    $K$, and each run returns $M$ of these approximations.
\item \textbf{Embarrassingly parallel variational inference
    (exact product)}
    (\texttt{EPVI\_exact})---The method introduced in this
    paper, which combines the $M$ subposteriors by computing
    all components of the product mixture density.
\item \textbf{Embarrassingly parallel variational inference
    (mixture product sampling)}
    (\texttt{EPVI\_sample})---The method introduced in this
    paper (Algorithm~\ref{sampleProductComponents}), which 
    samples from the product of the $M$ subposterior
    approximations.
\item \textbf{Embarrassingly parallel variational inference
    (sequential subset products)}
    (\texttt{EPVI\_subset})---The method introduced in this
    paper, which samples from products of pairs of
    subposteriors sequentially.
\end{itemize}
We do not aim to compare the benefits of VI in general in
this work, and therefore exclude comparisons against
alternative approximate inference methods such as MCMC,
expectation propagation, or other deterministic dynamics
inference algorithms (such as herding or Bayesian
quadrature). To assess the performance of our method, we
compute the log-likelihood of held-out test data given our
inferred variational approximation (which we can compute in
a consistent manner for many types of models).

Experiments were conducted with a standard cluster system.
We obtained subposterior variational approximations by
submitting batch jobs to each worker, since these jobs are
all independent. We then saved the results to the disk of
each worker and transfered them to the same machine, which
performed the product mixture sampling algorithms.  In each
of the following experiments involving timing, we first ran
the variational inference optimization procedures until
convergence (to provide a time for the
\texttt{Subposteriors} and \texttt{NVI} strategies).
Afterwards, we added the (maximum) time required for
learning the subposterior approximations, the time needed to
transfer the learned parameters to a master machine, and the
time required to run the product mixture sampling algorithms
(to provide a time for the \texttt{EPVI} methods).

\begin{figure*}[!ht]
    \includegraphics[width=0.93\textwidth]{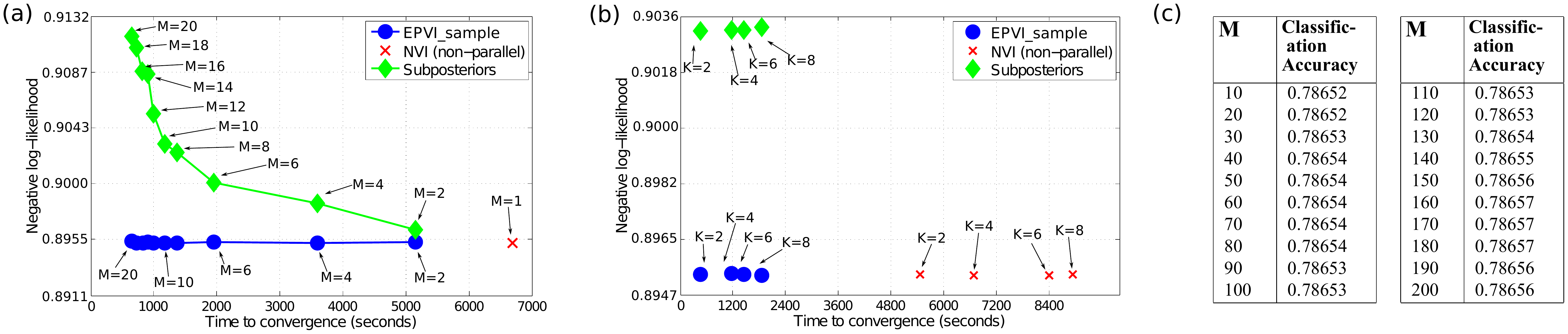}
    \caption{\label{fig:logregResults} Experimental results
    for hierarchical Bayesian logistic regression under
    varying numbers of (a) data-splits M and (b) NVI mixture
    components $K$. In (c) we show that the
    \texttt{EPVI\_sample} method maintains a consistent classification
    accuracy over a wide range of $M$.}
\end{figure*}

\begin{figure*}[!ht]
    \center{\includegraphics[width=1\textwidth]{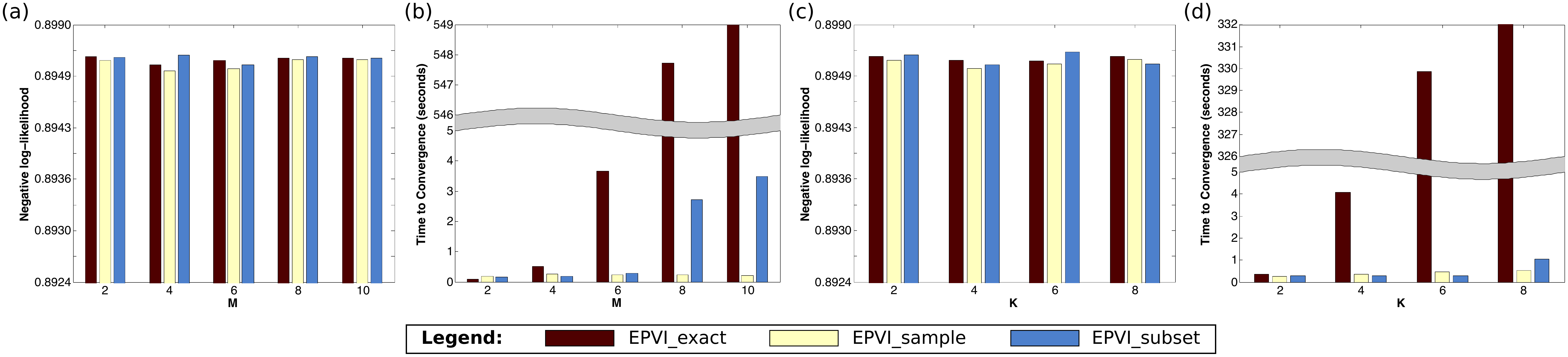}}
    \caption{\label{fig:approxComp} Comparison of the two
    product mixture sampling methods with the exact
    product mixture computation under (a)-(b) varying $M$ and
    (c)-(d) varying $K$.}
\end{figure*}

\vspace{-3mm}
\subsection{Bayesian Generalized Linear Models}
\vspace{-2mm}
Generalized linear models are widely used for a variety of
regression and classification problems. We use a
hierarchical Bayesian logistic regression model as a test
case in the following experiments. This model places a
Gaussian prior on a set of coefficients $\mathbf{w} \in
\mathbb{R}^V$ and draws class labels $\mathbf{y} \in
\mathbb{R}^N$, conditioned on the product of an observation
matrix $\mathbf{X} \in \mathbb{R}^{N \times V}$ and the
coefficients, passed through a logistic transform; further,
Gamma priors are placed on the variance parameter for each
coefficient. Notably, this model lacks conditional
conjugacy. We write the generative model as
\begin{enumerate}[nosep]
\item Draw global hyperparameter $\alpha \sim
    \text{Gamma}\left(a,b \right)$
\item For $v = 1,\ldots,V$, draw coefficient \\$w_v \sim
    \mathcal{N}\left(0,\alpha^{-1}\right)$
\item For $n = 1,\ldots,N$, draw observation \\$y_n \sim
    \text{Bernoulli}\left(\text{logit}^{-1}\left(-\mathbf{w}^\top
    \mathbf{x}_n \right) \right)$
\end{enumerate}
where $\mathbf{x}_n$ denotes the $n^{\text{th}}$ row of
$\mathbf{X}$. We partition the data by splitting
$\mathbf{X}$ and $\mathbf{y}$ into $M$ disjoint subsets each
of size $\frac{N}{M}$, and inferring a variational
approximation on each subset. This is illustrated in
Figure~\ref{fig:modelDiagram}(a).

\vspace{-1mm}
\textbf{Data.} We demonstrate our methods on the SUSY
particles
dataset\footnote{https://archive.ics.uci.edu/ml/datasets/SUSY},
in which the task is to classify whether or not a given
signal (measured by particle detectors in an accelerator)
will produce a supersymmetric particle. This dataset has
$N=5,000,000$ observations, of which we hold out $10\%$ for
evaluating the test log-likelihood.

\vspace{-1mm}
\textbf{Performance under varying $M$.} We vary the number
of data-splits $M$ from 2 to 20, and record the held-out
negative log-likelihood and time taken to converge for each
method.  For the \texttt{Subposteriors} result, we report
the maximum time taken to converge and the average negative
log-likelihood (over the set of $M$ subposteriors). We also
record the time taken to converge and the negative
log-likelihood for the \texttt{NVI} ($M=1$) standard VI
result. The number of mixture components for the NVI part of
all methods is fixed at $K=4$. We plot these results in
Figure~\ref{fig:logregResults}(a), and see that
\texttt{EPVI\_sample} reduces the time to convergence by
over an order of magnitude while maintaining nearly the same
test negative log-likelihood as \texttt{NVI}.  In
Figure~\ref{fig:logregResults}(c) we show that performance
of the \texttt{EPVI\_sample} method does not suffer
as we increase the number of machines over a greater range,
from $M=10$ to $M=200$. In this table, to give a more interpretable
view of performance, we show classification accuracy (on the
held out data) for each $M$. We see that classification 
accuracy stays nearly constant at approximately $0.7865$ 
as we increase $M$ throughout this range.

\begin{figure*}[!ht]
    \center{\includegraphics[width=1\textwidth]{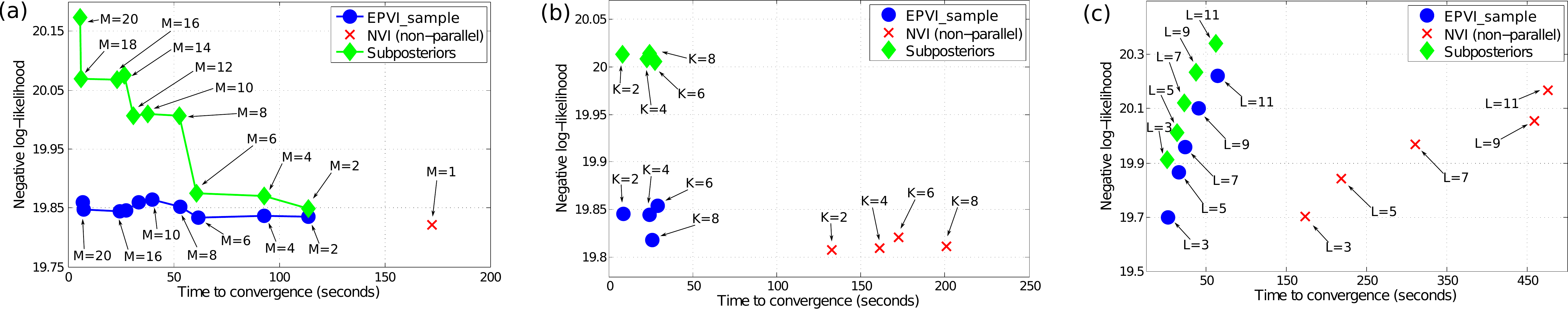}}
    \caption{\label{fig:nlmfResults} Experimental results
    for the nonlinear matrix factorization (topographical
    latent source analysis) model under varying numbers of
    (a) data-splits $M$, (b) mixture components $K$, and (c)
    latent sources $L$.}
\end{figure*}

\vspace{-1mm}
\textbf{Performance under varying $K$.} Next, we vary the
number of NVI mixture components $K$ from 2 to 8, and record
the held-out negative log-likelihood and time taken to
converge for each method. For parallel methods, we fix
$M=10$. We plot these results in
Figure~\ref{fig:logregResults}(b), and see that for all
values of $K$, \texttt{EPVI\_sample} decreases the time to
convergence by nearly tenfold, while maintaining virtually
identical test negative log-likelihood values.

\vspace{-1mm}
\textbf{Product mixture sampling methods.} We also conduct
experiments to judge the quality of our two product mixture
sampling procedures.  We aim to show that our methods yield
similar test log-likelihoods as computing expectations via
the exact product mixture while greatly decreasing the
computation time. We demonstrate this empirically over a
range of $M$ and $K$ values. Note that, since we need to
compare with the exact product, we restrict this range to
values in which we can compute all of the
(exponentially-many) product mixture components. For both
sampling methods, we fix $R=500$. Note that we perform the
$O(\log(M))$ rounds of \texttt{EPVI\_subset} sequentially on
the machine on which all samples are collected (not in
parallel).  We plot our results in
Figure~\ref{fig:approxComp}, and see that our sampling
methods yield very similar held-out negative log-likelihoods
as the exact product over all $M$
(Figure~\ref{fig:approxComp}(a)) and $K$
(Figure~\ref{fig:approxComp}(c)) values. We also see that
for roughly $M>6$ (Figure~\ref{fig:approxComp}(b)) and $K>4$
(Figure~\ref{fig:approxComp}(d)), the time needed to compute
the exact product increases substantially.  Additionally,
\texttt{EPVI\_sample} appears to fare slightly better than
\texttt{EPVI\_subset} in terms of both the test
log-likelihood and computation time.

\vspace{-4mm}
\subsection{Nonlinear Matrix Factorization}
\vspace{-2mm}
We next apply our algorithm to a nonlinear matrix
factorization model known as topographic latent source
analysis (TLSA) \cite{gershman2012nonparametric}.  This
model can be viewed as representing an observed matrix as a
covariate-dependent superposition of $L$ latent sources. In
particular, an observed matrix $\mathbf{U} \in \mathbb{R}^{N
\times V}$ is assumed to be drawn conditioned on an observed
matrix of covariates $\mathbf{X} \in \mathbb{R}^{N \times
C}$, an inferred weight matrix $\mathbf{W} \in \mathbb{R}^{C
\times L}$, and a basis matrix $\mathbf{G} \in \mathbb{R}^{L
\times V}$ constructed by evaluating a parameterized spatial
basis function with parameters
$\{\bar{\mathbf{r}}_l,\lambda_l \}$, written $g_{lv} =
\exp\{\lambda_l^{-1} \|\mathbf{r}_l -
\bar{\mathbf{r}}_l\|^2\}$.  Similar to the previous model,
this model lacks conditional conjugacy. We can write the
full generative process as
\begin{enumerate}[nosep]
\item For latent source $l=1,\ldots,L$,
\begin{enumerate}[nosep]
    \item Draw hyperparameter $\lambda_l \sim
        \text{Exponential}(\rho)$
    \item For $d = 1,\ldots,M$, draw $\bar{r}_{ld} \sim
        \text{Beta}(1,1)$
    \item For $c = 1,\ldots,C$, draw $w_{cl} \sim
        \mathcal{N}(0,\sigma_w^2)$
\end{enumerate}
\item For $n = 1,\ldots,N$,
\begin{enumerate}[nosep]
    \item For $v = 1, \ldots, V$, draw observation \\$u_{nv}
        \sim \mathcal{N}\left(\sum_{c=1}^C x_{nc} \sum_{l=1}^L
        w_{cl}g_{lv}, \tau^{-1} \right)$
\end{enumerate}
\end{enumerate}
We partition the data by splitting observed matrices
$\mathbf{U}$ and $\mathbf{X}$ into $M$ disjoint subsets each
of size $\frac{N}{M}$, and inferring a variational
approximation on each subset. This is illustrated in
Figure~\ref{fig:modelDiagram}(b).

\vspace{-1mm}
\textbf{Data.} In the following experiments, we generate $N
= 1,000$ observations in $V = 50$ dimensions by choosing
hyperparameters $\{\tau=1,\text{ }\sigma_w^2=5,\text{
}\rho=1\}$, and drawing from the above generative process.
We hold out $10\%$ of the data for evaluating the test
log-likelihood.

\vspace{-1mm}
\textbf{Performance under varying $M$, $K$, and $L$.}
Similar to the previous model, we first conduct experiments
showing held-out negative log-likelihood versus time under
varying values for the number of data-splits $M$ and NVI
mixture components $K$. These results are shown in
Figure~\ref{fig:nlmfResults}(a)-(b). We see that
\texttt{EPVI\_sample} reduces the time to convergence
(particularly as the number of subposteriors $M$ increases)
while maintaining a similar test negative log-likelihood as
\texttt{NVI}.  We also evaluate the performance of our
method under different numbers of latent sources $L$. We
vary $L$ from 2 to 8, and record the held-out negative
log-likelihood and time taken to converge, and again see
positive results (Figure~\ref{fig:nlmfResults}(c)).

\vspace{-4mm}
\section{Conclusion}
\vspace{-3mm}
In this paper, we developed an embarrassingly parallel VI
algorithm for Bayesian inference in a distributed setting,
that does not require models with conditional conjugacy and
exponential family assumptions.
Unlike existing methods,
our strategy uses a decomposition of the full-data posterior
involving a product of subposterior densities, which was
recently developed in the parallel MCMC literature. We have
shown promising empirical results on nonconjugate
models, which illustrate the ability of our method to
perform VI in a distributed setting, and provide large
speed-ups, while maintaining an accurate posterior
approximation.

\bibliography{main}
\bibliographystyle{plain} 

\end{document}